\def\year{2020}\relax
\documentclass[letterpaper]{article} 
\usepackage{aaai20}  
\usepackage{times}  
\usepackage{helvet} 
\usepackage{courier}  
\usepackage[hyphens]{url}  
\usepackage{graphicx} 
\usepackage{subcaption}
\usepackage{amsthm}

\usepackage{graphicx}  
\usepackage{booktabs}       
\usepackage{amsmath}
\usepackage{amssymb}
\usepackage{mathtools}
\usepackage{float}
\usepackage{multicol}
\newcommand{\citet}[1]{\citeauthor{#1} \shortcite{#1}}

\newcommand{\Expect}{\mathbb{E}}
\newcommand{\KL}{\text{D}_\text{KL}}

\newcommand{\X}{\mathcal{X}}

\newcommand{\Z}{\mathcal{Z}}
\newcommand{\Y}{\mathcal{Y}}

\newcommand{\sset}[1]{\{#1\}}
\renewcommand{\L}{\mathcal{L}}
\DeclareMathOperator*{\minimize}{min}

\newcommand{\brac}[1]{\left[#1\right]}
\theoremstyle{definition}
\newtheorem{definition}{Definition}
\newtheorem{theorem}{Theorem}
\newtheorem{lemma}[theorem]{Lemma}

\title{Virtual Mixup Training for Unsupervised Domain Adaptation}
\makeatletter
    \renewcommand{\copyright@on}{F}
\makeatother

\author{%
Xudong Mao$^{\textbf{1}}$\thanks{indicates equal contribution} \ Yun Ma$^{\textbf{1}*}$ \ Zhenguo Yang$^{\textbf{1}}$ \ Yangbin Chen$^\textbf{2}$\  Qing Li$^\textbf{1}$\\
    $^1$Department of Computing, The Hong Kong Polytechnic University\\
    $^2$Department of Computer Science, City University of Hong Kong\\
    \texttt{xudong.xdmao@gmail.com mayun371@gmail.com} \\
    \texttt{yzgcityu@gmail.com robinchen2-c@my.cityu.edu.hk} \\
    \texttt{csqli@comp.polyu.edu.hk} \\
}

\begin{document}

\maketitle

\begin{abstract}
 We study the problem of unsupervised domain adaptation which aims to adapt models trained on a labeled source domain to a completely unlabeled target domain. Recently, the cluster assumption has been applied to unsupervised domain adaptation and achieved strong performance. One critical factor in successful training of the cluster assumption is to impose the locally-Lipschitz constraint to the model. Existing methods only impose the locally-Lipschitz constraint around the training points while miss the other areas, such as the points in-between training data. In this paper, we address this issue by encouraging the model to behave linearly in-between training points. We propose a new regularization method called Virtual Mixup Training (VMT), which is able to incorporate the locally-Lipschitz constraint to the areas in-between training data. Unlike the traditional mixup model, our method constructs the combination samples without using the label information, allowing it to apply to unsupervised domain adaptation. The proposed method is generic and can be combined with most existing models such as the recent state-of-the-art model called VADA. Extensive experiments demonstrate that VMT significantly improves the performance of VADA on six domain adaptation benchmark datasets. For the challenging task of adapting MNIST to SVHN, VMT can improve the accuracy of VADA by over 30\%. Code is available at \url{https://github.com/xudonmao/VMT}.
\end{abstract}

\section{Introduction}
\label{sec:intro}
Deep neural networks have launched a profound reformation in a wide variety of fields such as image classification \cite{Krizhevsky2012}, detection \cite{Girshick2014}, and segmentation \cite{Long2015_seg}. However, the performance of deep neural networks is often based on large amounts of labeled training data. In real-world tasks, generating labeled training data can be very expensive and may not always be feasible. One approach to this problem is to learn from a related labeled source data and generalize to the unlabeled target data, which is known as domain adaptation. In this work, we consider the problem of unsupervised domain adaptation where the training samples in the target domain are completely unlabeled.

For unsupervised domain adaptation, \citet{Ganin2016} proposed the domain adversarial training to learn domain-invariant features between the source and target domains, which has been a basis for numerous domain adaptation methods \cite{Tzeng2017,Kumar2018,Shu2018,Saito2018,Xie2018}. Most of the follow-up studies focus on how to learn better-aligned domain-invariant features, including the approaches of adversarial discriminative adaptation \cite{Tzeng2017}, maximizing classifier discrepancy \cite{Saito2018}, and class conditional alignment \cite{Xie2018,Kumar2018}.

\begin{figure*}[t]
\begin{subfigure}[b]{0.21\textwidth}
    \includegraphics[width=\textwidth]{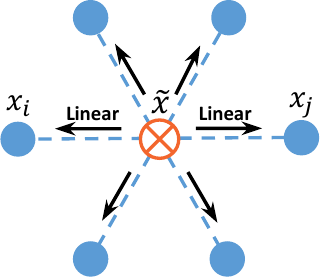}
    \caption{}
\end{subfigure}
\hfill
\begin{subfigure}[b]{0.75\textwidth}
    \includegraphics[width=\textwidth]{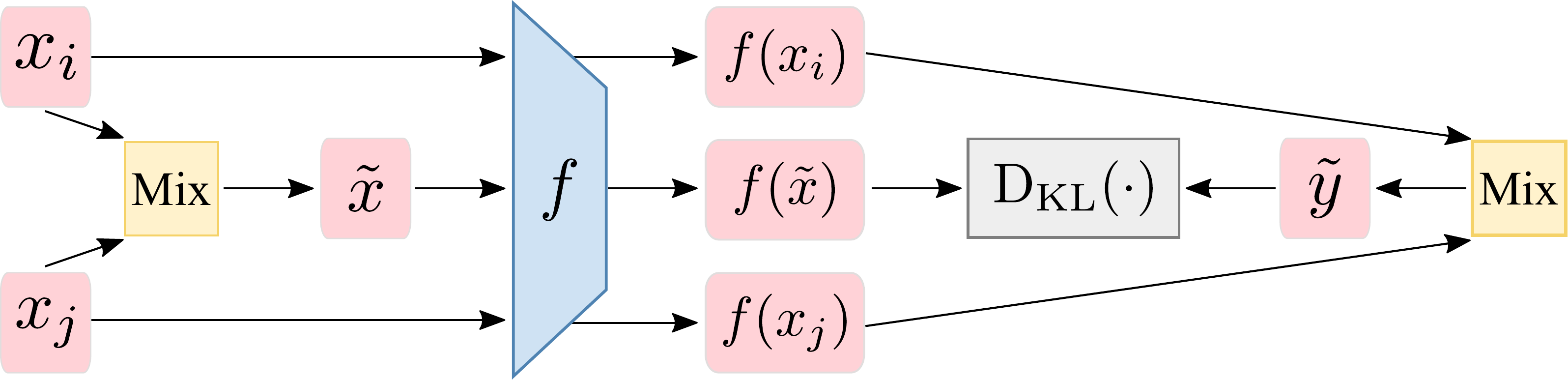}
    \caption{}
\end{subfigure}
\caption{
(a) Illustration of VMT, where blue points denote the training samples and the center point is in the area in-between training samples. For the center point, different pairs of training samples make it behave linearly in different directions, thus regularizing the center point to be locally-Lipschitz. (b) The framework of VMT, where $f$ is a classifier and $\KL(\cdot)$ denotes the KL-divergence.
}

\label{fig:intro}
\end{figure*}

Recently, \citet{Shu2018} have successfully incorporated the cluster assumption \cite{Grandvalet2005} into the framework of domain adversarial training by adopting the conditional entropy loss. They also pointed out that the locally-Lipschitz constraint is critical to the performance of the cluster assumption. Without the locally-Lipschitz constraint, the classifier may abruptly change its predictions in the vicinity of the training samples. This will push the decision boundary close to the high-density regions, which violates the cluster assumption. To this end, they adopted virtual adversarial training \cite{Miyato2018} to constrain the local Lipschitzness of the classifier. 

However, virtual adversarial training only incorporates the locally-Lipschitz constraint to the training points while misses the other areas. To this end, we propose a new regularization method called Virtual Mixup Training (VMT), which is able to regularize the areas in-between training points to be locally-Lipschitz. In general, VMT extends mixup \cite{Zhang2018} by replacing the real labels with virtual labels (i.e., the predicted labels by the classifier), allowing it to apply to unsupervised domain adaptation. Our proposed method is based on the fact that mixup favors linear behavior in-between training samples \cite{Zhang2018}. The idea of VMT is simple yet powerful: as Figure \ref{fig:intro}(a) shows, for the center point, different pairs of training samples make it behave linearly in different directions, and we can easily verify that if the model is linear in all directions at the center point, then the model is locally-Lipschitz at this center point. On the other hand, the locally-Lipschitz constraint is beneficial for pushing the decision boundary away from the high-density regions, which favors the cluster assumption \cite{Ben-David2010,Shu2018}. 

Specifically, as shown in Figure \ref{fig:intro}(b), we first construct convex combinations, denoted as $(\tilde{x}, \tilde{y})$, of pairs of training samples and their virtual labels, and then define a penalty term that punishes the difference between the combined sample's prediction $f(\tilde{x})$ and the combined virtual label $\tilde{y}$. This penalty term encourages a linear change of the output distribution in-between training samples. 

In practice, \citet{Shu2018} pointed out that the conditional entropy loss sometimes behaves unstably and even finds a degenerate solution for some challenging tasks, which also occurs in our model. To tackle this problem, we propose to mixup on the logits (i.e., the input of the softmax layer) instead of the probabilities (i.e., the output of the softmax layer). We argue that the problem of mixup on probabilities is that most of the probability values will tend to be zero since the targets are one-hot vectors, while the logits still have non-zero values.

In the experiments, we combine VMT with a recent state-of-the-art model called VADA \cite{Shu2018}, and evaluate on six commonly used benchmark datasets. The experimental results show that VMT is able to improve the performance of VADA for all tasks. For the most challenging task, MNIST $\rightarrow$ SVHN without instance normalization, our model improves over VADA by 33.6\%.

Our contributions can be summarized as follows:
\begin{itemize}
\item We propose the Virtual Mixup Training (VMT), a new regularization method, to impose the locally-Lipschitz constraint to the areas in-between training data.
\item We propose to mixup on logits rather than mixup on probabilities to further improve the performance and training stability of VMT.
\item We evaluate VMT on six benchmark datasets, and the experimental results demonstrate that VMT can achieve state-of-the-art performance on all the six datasets.
\end{itemize}

\section{Related Work}

\subsubsection{Domain Adaptation}
Domain adaptation has gained extensive attention in recent years due to its advantage of utilizing unlabeled data. A theoretical analysis of domain adaptation was presented in \cite{Ben-David2010}. Early works \cite{Shimodaira2000} tried to minimize the discrepancy distance between the source and target feature distributions. \citet{Long2015} and Sun \& Saenko (\citeyear{Sun2016}) extended this method by matching higher-order statistics of the two distributions. \citet{Huang2007}, \citet{Tzeng2015}, and \citet{Ganin2016} proposed to project the source and target feature distributions into some common space and match the learned features as close as possible. Specifically, \citet{Ganin2016} proposed the domain adversarial training to learn domain-invariant features, which has been a basis of numerous domain adaptation methods \cite{Tzeng2017,Saito2018,Xie2018,Shu2018,Kumar2018}. \citet{Tzeng2017} generalized a framework based on domain adversarial training and proposed to combine the discriminative model and GAN loss \cite{Goodfellow2014}. \citet{Saito2018} proposed to utilize two different classifiers to learn not only domain-invariant but also class-specific features. \citet{Shu2018} proposed to combine the cluster assumption \cite{Grandvalet2005} with domain adversarial training. They also adopted virtual adversarial training \cite{Miyato2018} to constrain the local Lipschitzness of the classifier, as they found that the locally-Lipschitz constraint is critical to the performance of the cluster assumption. \citet{Kumar2018} extended \cite{Shu2018} to align class-specific features by using co-regularization \cite{Sindhwani2005}. We also follow the line of \cite{Shu2018} and propose a new method to impose the locally-Lipschitz constraint to the areas in-between training data. There are also many other promising models including domain separation networks \cite{Konstantinos2016}, reconstruction-classification networks \cite{Ghifary2016}, tri-training \cite{Saito2017}, self-ensembling \cite{French2018}, and image-to-image translation \cite{Bousmalis2017}.

\subsubsection{Local Lipschitzness}
\citet{Grandvalet2005} pointed out that the local Lipschitzness is critical to the performance of the cluster assumption. \citet{Ben-David2014} also showed in theory that Lipschitzness can be viewed as a way of formalizing the cluster assumption. Constraining local Lipschitzness has been proven as effective in semi-supervised learning \cite{Sajjadi2016,Laine2017,Miyato2018} and domain adaptation \cite{French2018,Shu2018}. In general, these methods smooth the output distribution of the model by constructing surrounding points of the training points and enforcing consistent predictions between the surrounding and training points. Specifically, \citet{Sajjadi2016}, and Laine \& Aila (\citeyear{Laine2017}) utilized the randomness of neural networks to construct the surrounding points. \citet{French2018} proposed to construct two different networks and enforce the two networks to output consistent predictions for the same input. \citet{Miyato2018} utilized the adversarial examples \cite{Goodfellow2015} to regularize the model from the direction violating the local Lipschitzness mostly.

\subsubsection{Mixup}
\citet{Zhang2018} proposed a regularization method called mixup to improve the generalization of neural networks. Mixup generates convex combinations of pairs of training examples and their labels, favoring the smoothness of the output distribution. A similar idea was presented in \cite{Tokozume2018} for image classification. \citet{Verma2018} extended mixup by mixing on the output of a random hidden layer. \citet{Guo2019} proposed to learn the mixing policy by an additional network instead of the random policy. A similar idea to ours was described in \cite{Verma2019} for semi-supervised learning. They also used mixup to provide consistent predictions between unlabeled training samples. \citet{Berthelot2019} extended this method by mixing between the labeled and unlabeled samples.

\subsubsection{Virtual Labels}
Virtual (or pseudo) labels have been widely used in semi-supervised learning \cite{Miyato2018} and domain adaptation \cite{Chen2011,Saito2017,Xie2018}. In particular, \citet{Chen2011} and \citet{Saito2017} proposed to first use multiple classifiers to assign virtual labels to the target samples, and then train the classifier using the target samples with virtual labels. \citet{Xie2018} proposed to calculate the class centroids of the virtual labels to reduce the bias caused by the false virtual labels. The most related method to ours is virtual adversarial training \cite{Miyato2018}. Virtual adversarial training enforces the virtual labels of the original sample and its adversarial example to be similar, which can be used to impose the locally-Lipschitz constraint to the training samples.

\section{Preliminaries}
\subsection{Domain Adversarial Training}
We first describe domain adversarial training \cite{Ganin2016} which is a basis of our model. Let $\X_s$ and $\Y_s$ be the distributions of the input sample $x$ and label $y$ from the source domain, and let $\X_t$ be the input distribution of the target domain. Suppose a classifier $f = h \circ g$ can be decomposed into a feature encoder $g$ and an embedding classifier $h$. The input $x$ is first mapped through the feature encoder $g: \X \to \Z$, and then through the embedding classifier $h: \Z \to \Y$. On the other hand, a domain discriminator $d: \Z \to (0,1)$ maps the feature vector to the domain label $(0,1)$. The domain discriminator $d$ and feature encoder $g$ are trained adversarially: $d$ tries to distinguish whether the input sample $x$ is from the source or target domain, while $g$ aims to generate indistinguishable feature vectors of samples from the source and target domains. The objective of domain adversarial training can be formalized as follows: 
\begin{align}
\label{eq:dann}
\begin{split}
\minimize_{f}&  \L_y(f;\X_s, \Y_s) + \lambda_d \L_d(g;\X_s,\X_t), \\
&\L_y(f;\X_s, \Y_s) = -\Expect_{(x, y) \sim (\X_s, \Y_s)} \brac{ y^\top \ln f(x)}, \\
&\L_d(g;\X_s,\X_t) = \sup_{d}\Expect_{x \sim \X_s} \brac{\ln d(g(x))} + \\
&\ \ \ \ \ \ \ \ \ \ \ \ \ \ \ \ \ \ \ \ \ \ \ \ \ \ \ \ \ \ \ \ \ \ \Expect_{x \sim \X_t} \brac{\ln (1 - d(g(x)))},
\end{split}
\end{align}
where $\lambda_d$ is used to adjust the weight of $\L_d$.

\subsection{Cluster Assumption}
\label{sec:cluster}
The cluster assumption states that the input data contains clusters, and if samples are in the same cluster, they come from the same class \cite{Grandvalet2005}. It has been widely used in semi-supervised learning \cite{Grandvalet2005,Sajjadi2016,Miyato2018}, and recently has been applied to unsupervised domain adaptation \cite{Shu2018}. The conditional entropy minimization is usually adopted to enforce the behavior of the cluster assumption \cite{Grandvalet2005,Sajjadi2016,Miyato2018,Shu2018}:
\begin{align}
\label{eq:entropy}
\L_c(f; \X_t) = -\Expect_{x \sim \X_t} \brac{f(x)^\top \ln f(x)}.
\end{align}

\subsection{Local Lipschitzness}
We first recall the definition of local Lipschitzness:
\theoremstyle{definition}
\begin{definition}{(Local Lipschitzness)}
We say a function $f: \X \to \Y$ is locally-Lipschitz, if for each $x_0 \in \X$, there exists constants $L>0$ and $\delta_0>0$ such that $||f(x)-f(x_0)||\leq L||x-x_0||$ holds for all $||x-x_0||<\delta_0$, $x \in \X$.
\label{def:local}
\end{definition}
\citet{Shu2018} pointed out that the locally-Lipschitz constraint is critical to the performance of the cluster assumption, and adopted virtual adversarial training \cite{Miyato2018} to impose the locally-Lipschitz constraint:
\begin{align}
\L_v(f; \X) = \Expect_{x \sim \X} \brac{\max_{\| r\| \le \epsilon} \KL(f(x) \| f(x + r))}.
\label{eq:vat}
\end{align}

\section{Virtual Mixup Training}
\citet{Shu2018} adopted two methods, including the conditional entropy (Eq. \ref{eq:entropy}) and virtual adversarial training (Eq. \ref{eq:vat}), to enforce the cluster assumption. Minimizing the conditional entropy forces the classifier to be confident on the training points, thus driving the decision boundary away from the high-density regions. On the other hand, virtual adversarial training imposes the locally-Lipschitz constraint to the training points, and the locally-Lipschitz constraint can favor the cluster assumption, because it prevents the classifier to abruptly change its predictions in the vicinity of the training points, thus driving the decision boundary away from the high-density regions. However, both the conditional entropy and virtual adversarial training only consider the areas around the training points while miss the other areas such as the points in-between training data. To remedy this problem, we propose the Virtual Mixup Training (VMT), which is able to incorporate the locally-Lipschitz constraint to the areas in-between training data. 

The original mixup \cite{Zhang2018} model has shown the ability to enforce the classifier to behave linearly in-between training samples by applying the following convex combinations of labeled samples:
\begin{align}
\begin{split}
\label{eq:mixup}
  \tilde{x} &= \lambda x_i + (1 - \lambda) x_j,\\
  \tilde{y} &= \lambda y_i + (1 - \lambda) y_j.
\end{split}
\end{align}
However, for unsupervised domain adaptation, we have no direct information about $y_i$ and $y_j$ of the target domain. Inspired by \cite{Miyato2018}, we replace $y_i$ and $y_j$ with the approximations, $f(x_i)$ and $f(x_j)$, which are the current predictions by the classifier $f$. Literally, we call $f(x_i)$ and $f(x_j)$ virtual labels, and formalize our proposed VMT as follows:
\begin{align}
\label{eq:vmt}
\begin{split}
  \tilde{x} &= \lambda x_i + (1 - \lambda) x_j,\\
  \tilde{y} &= \lambda f(x_i) + (1 - \lambda) f(x_j),
\end{split}
\end{align}
where $\lambda \sim \text{Beta}(\alpha, \alpha)$, for $\alpha \in (0, \infty)$. Then we penalize the difference between the prediction $f(\tilde{x})$ and the virtual label $\tilde{y}$:
\begin{align}
\label{eq:vmt_loss}
\L_m(f; \X) = \Expect_{x \sim \X} \brac{ \KL(\tilde{y} \| f(\tilde{x}))}. 
\end{align}

Like the original mixup, our proposed VMT also encourages the classifier $f$ to behave linearly between $x_i$ and $x_j$. Moreover, as shown in Figure \ref{fig:intro}(a), for the virtual point $\tilde{x}$, different pairs of $x_i$ and $x_j$ enforce it to behave linearly in different directions, and we can verify that if $f$ is linear in all directions at $\tilde{x}$, then $f$ is locally-Lipschitz at $\tilde{x}$. We prove the linear behavior and the local Lipschitzness of VMT in the Appendix. We also empirically verify this by plotting in Figure \ref{fig:grad} the gradient norms of VADA and VMT in-between training samples, since the gradient norm is an indicator of the local Lipschitzness \cite{Hein2017}. It shows that VMT has much smaller gradient norms than VADA.

The proposed VMT is generic and can be combined with most existing models. In particular, we can combine VMT with a recent state-of-the-art model, VADA, which then leads us to the following objective:
\begin{align}
\label{eq:final_obj}
\begin{split}
&\minimize_{f} \L_{y,d} + \lambda_s [\L_m(f; \X_s) + \L_v(f; \X_s)]+ \\
&\ \ \ \ \ \ \ \ \ \ \ \ \ \ \ \ \ \ \ \ \lambda_t [\L_m(f; \X_t)+ \L_v(f; \X_t)+\L_c(f; \X_t) ], \\
\end{split}
\end{align}
where $\L_{y,d} = \L_y(f;\X_s, \Y_s) + \lambda_d \L_d(g;\X_s,\X_t)$ and $(\lambda_s, \lambda_t)$ are used to adjust the weights of the penalty terms on the source and target domains. Note that we also incorporate virtual adversarial training in Eq. \ref{eq:final_obj}, and empirically show in the experiments that VMT is orthogonal to virtual adversarial training for most tasks. For the source domain, we also replace $y_i$ and $y_j$ with the virtual labels, without using the label information.

Like the original mixup, the implementation of VMT is also simple and straightforward. One important advantage of VMT is the low computational cost, and we show in the experiments that VMT has a much lower computational cost than virtual adversarial training. Despite its simplicity, VMT achieves a new state-of-the-art performance on several benchmark datasets.

\begin{figure}[t]
\centering
 \includegraphics[width=0.3\textwidth]{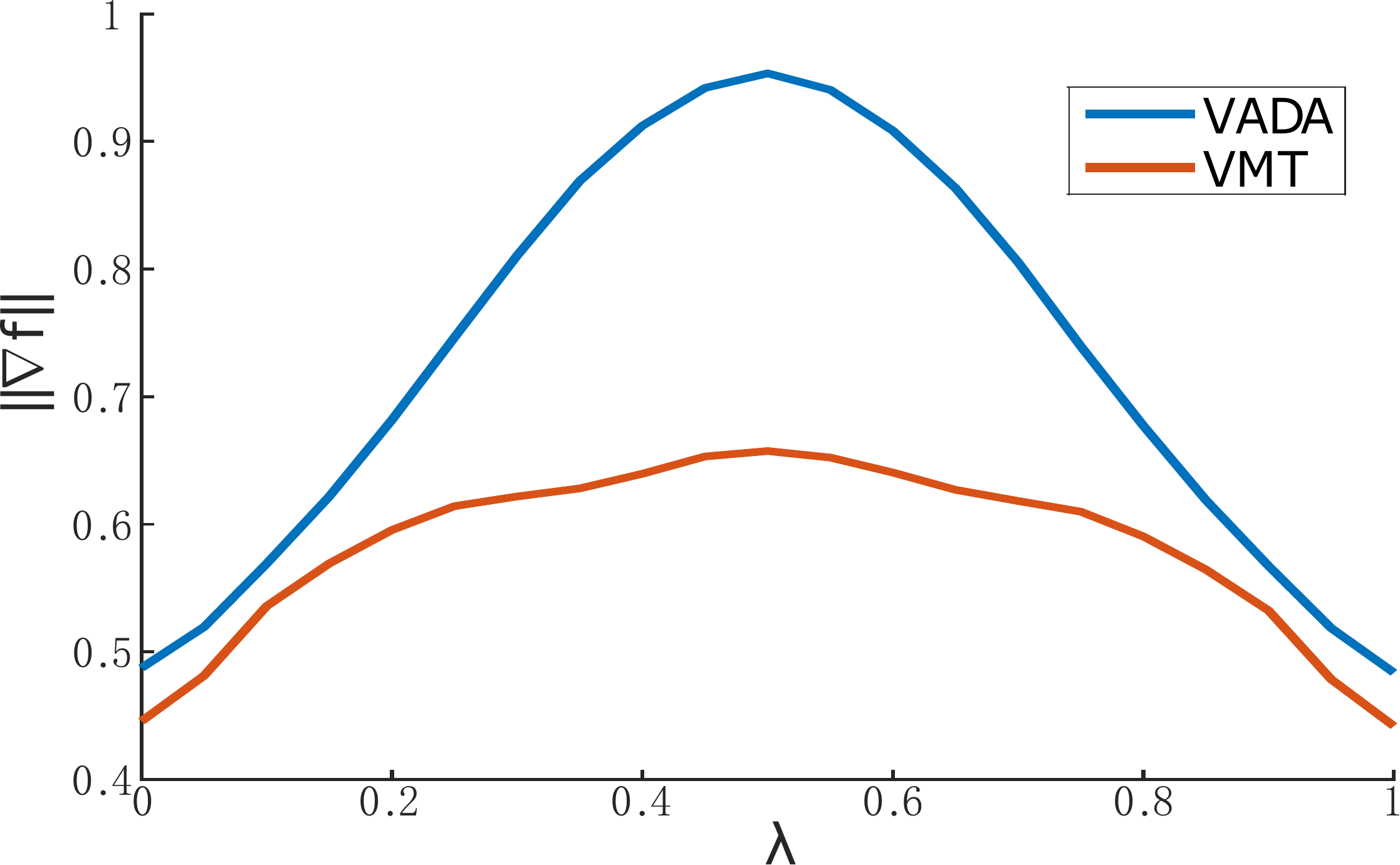}
\caption{Gradient norms of the classifier $f$ with respect to the input in-between training samples: $\tilde{x} = \lambda x_i + (1 - \lambda) x_j$. Both models are trained with the same architecture and evaluated on the whole test set. VMT has much smaller gradient norms than VADA.}
\label{fig:grad}
\end{figure}

\subsection{Mixup on Logits}

As stated in \cite{Shu2018}, VADA shows high-variance results for some tasks and even finds a degenerate solution quickly. In practice, we also observe that Eq. \ref{eq:vmt} sometimes collapse to a degenerate solution. To tackle this problem, we propose to mixup-on-logits (i.e., the input of the softmax layer) instead of mixup-on-probabilities (i.e., the output of the softmax layer), which is also studied in literature \cite{Varga2017} for supervised learning. Let $f_{\text{logits}}$ denote the layers before the softmax layer and $f_{\text{softmax}}$ denote the softmax layer. Then Eq. \ref{eq:vmt} is modified as:
\begin{align}
\label{eq:vmt_logits}
\begin{split}
  \tilde{y} = f_{\text{softmax}}(\lambda f_{\text{logits}}(x_i) + (1 - \lambda) f_{\text{logits}}(x_j)).
\end{split}
\end{align}
The problem of mixup-on-probabilities is that most of the probabilities will tend to be zero since the targets are one-hot vectors, and mixup between zeros vanishes the effect of favoring linear behavior. This problem does not arise in mixup-on-logits because the logits still have non-zero values even though the probabilities are close to zero. In the experiments, we empirically verify that mixup-on-logits performs more stably than mixup-on-probabilities, especially for the challenging task of MNIST $\rightarrow$ SVHN. 

\begin{table*}[t]
\centering
\begin{tabular}{l|cccccc}
\toprule
\multicolumn{1}{r|}{Source}             & MNIST         & SVHN          & MNIST     & SYN        & CIFAR         & STL   \\
\multicolumn{1}{r|}{Target}             & SVHN          & MNIST         & MNIST-M   & SVHN       & STL           & CIFAR \\
\midrule
MMD \cite{Long2015}                     &    -          & $71.1$        & $76.9$    & $88.0$     &    -          &    -  \\ 
DANN \cite{Ganin2016}                   & $35.7$        & $71.1$        & $81.5$    & $90.3$     &    -          &    -   \\
DRCN \cite{Ghifary2016}                 & $40.1$        & $82.0$        &    -      &    -       & $66.4$        & $58.7$ \\
DSN \cite{Bousmalis2016}                &    -          & $82.7$        & $83.2$    & $91.2$     &    -          &    -   \\
kNN-Ad \cite{Sener2016}                 & $40.3$        & $78.8$        & $86.7$    &    -       &    -          &    -   \\    
PixelDA \cite{Bousmalis2017}            & -             & -             & $98.2$    & -          & -             & -      \\    
ATT \cite{Saito2017}                    & $52.8$        & $86.2$        & $94.2$    & $92.9$     &    -          &    -   \\    
$\Pi$-model (aug) \cite{French2018}     & $71.4$        & $92.0$        &    -      & $94.2$     & $76.3$        & $64.2$ \\
\midrule
\multicolumn{7}{c}{Without Instance-Normalized Input:}\\
\midrule
Source-Only                             & $27.9$        & $77.0$        & $58.5$    & $86.9$     & $76.3$        & $63.6$ \\
\midrule                                                                
VADA \cite{Shu2018}                     & $47.5$        & $97.9$        & $97.7$    & $94.8$     & $80.0$        & $73.5$  \\    
Co-DA \cite{Kumar2018}                  & $55.3$        & $\bf{98.8}$   &$\bf{99.0}$& $96.1$     & $81.4$        & $76.4$  \\    
VMT (ours)                              & $\bf{59.3}$   &$\bf{98.8}$    &$\bf{99.0}$& $\bf{96.2}$& $\bf{82.0}$   & $\bf{80.2}$  \\    
\midrule                                                                
VADA + DIRT-T  \cite{Shu2018}             & $54.5$        & $99.4$        & $98.9$    & $96.1$     & -             & $75.3$ \\
Co-DA + DIRT-T \cite{Kumar2018}           & $63.0$        & $99.4$        &$\bf{99.1}$&$\bf{96.5}$ & -             & $77.6$  \\  
VMT + DIRT-T (ours)                       & $\bf{88.1}$  &$\bf{99.5}$     &$\bf{99.1}$&$\bf{96.5}$  & -            & $\bf{80.6}$  \\    
\midrule
\multicolumn{7}{c}{With Instance-Normalized Input:}\\
\midrule
Source-Only                             & $40.9$        & $82.4$        & $59.9$    & $88.6$     & $77.0$        & $62.6$ \\
\midrule                                                                
VADA \cite{Shu2018}                     & $73.3$        & $94.5$        & $95.7$    & $94.9$     & $78.3$        & $71.4$ \\
Co-DA \cite{Kumar2018}                  & $81.7$        & $98.7$        &$\bf{98.0}$& $96.0$     & $80.6$        & $74.7$  \\    
VMT (ours)                              & $\bf{85.2}$  & $\bf{98.9}$    &$\bf{98.0}$& $\bf{96.4}$& $\bf{81.3}$   & $\bf{79.5}$  \\
\midrule
VADA + DIRT-T \cite{Shu2018}              & $76.5$      & $\bf{99.4}$          & $98.7$    & $96.2$     & -             & $73.3$ \\
Co-DA + DIRT-T \cite{Kumar2018}           & $88.0$      & $\bf{99.4}$          & $98.8$    & $96.5$     & -             & $75.9$  \\  
VMT + DIRT-T (ours)                       & $\bf{95.1}$ & $\bf{99.4}$     &$\bf{98.9}$& $\bf{96.6}$& -             & $\bf{80.2}$  \\
\midrule
\end{tabular}
\caption{Test set accuracy on the visual domain adaptation benchmark datasets. For all tasks, VMT improves the accuracy of VADA and achieves the state-of-the-art performance.} 
\label{table:evaluation}
\end{table*}

\section{Experiments}
\label{sec:exp}
In our experiments, we focus on the visual domain adaptation and evaluate our model on six benchmark datasets including MNIST, MNIST-M, Synthetic Digits (SYN), Street View House Numbers (SVHN), CIFAR-10, and STL-10.

\subsection{Iterative Refinement Training}
Following VADA, we also perform an iterative refinement training technique called DIRT-T  \cite{Shu2018} for further optimizing the cluster assumption on the target domain. Specifically, we first initialize with a trained VMT model using Eq. \ref{eq:final_obj}, and then iteratively minimize the following objective on the target domain:
\begin{align}
\minimize_{f_n} \lambda_t \L_t(f_n;\X_t) + \beta\Expect\brac{\KL(f_{n-1}(x) \| f_n(x))},
\end{align}
where $\L_t = \L_m + \L_v + \L_c$. In practice, for the initialization model, we do not need to train the VMT model until convergence, and pre-training for 40000 or 80000 iterations is enough to achieve good performance. We report the results of using and without using DIRT-T in the following experiments.

\subsection{Hyperparameters}
Following \cite{Shu2018}, we tune the four hyperparameters $(\lambda_d,\lambda_s,\lambda_t,\beta)$ by randomly selecting 1000 labeled target samples from the training set as the validation set. We restrict the hyperparameter search to $\lambda_d = \sset{0, 0.01}$, $\lambda_s = \sset{0, 1}$, $\lambda_t = \sset{0.01,0.02,0.04,0.06,0.08,0.1,1}$, and $\beta = \sset{0.001,0.01,0.1,1}$; $\alpha$ in Eq. \ref{eq:vmt} is fixed as $1$ for all experiments. A complete list of the hyperparameters is presented in the Appendix.

Compared with VADA, the main different setting is $\lambda_t$. We empirically find that increasing the value of $\lambda_t$ is able to improve the performance significantly. But for VADA, it will collapse to a degenerate solution if we set the same value of $\lambda_t$.

\begin{table*}[t]
\centering
\begin{tabular}{c|cccccccccc|c}
\midrule
\multicolumn{11}{c|}{MNIST $\rightarrow$ SVHN without Instance-Normalized Input:} & Average\\
 \midrule
 VMT@40K & $57.7$ & $57.6$ & $57.4$ & $50.1$ & $54.1$ & $50.0$& $52.7$& $40.1$& $53.9$& $45.5$ & $51.9 \pm 5.7$\\
\midrule
 VMT@160K & $64.6$ & $64.6$ & $58.3$ & $61.4$ & $57.3$ & $63.7$& $64.5$& $49.4$& $56.9$& $52.0$ & $59.3 \pm 5.5$\\
\midrule
 VMT + DIRT-T@160K & $95.9$ & $95.9$ & $95.3$ & $95.3$ & $95.1$ & $83.6$& $82.8$& $80.4$& $79.9$& $76.6$ & $88.1 \pm 8.0$\\
\midrule
\multicolumn{11}{c|}{MNIST $\rightarrow$ SVHN with Instance-Normalized Input:} & Average\\
\midrule
 VMT@40K         & $86.1$ & $85.6$ & $86.1$ & $85.8$ & $85.2$ & $85.6$& $84.3$& $84.0$& $84.5$& $84.3$ & $85.2 \pm 0.8$\\
\midrule
 VMT + DIRT-T@160K  & $96.0$ & $95.9$ & $95.6$ & $95.5$ & $95.4$ & $95.3$& $95.3$& $95.1$& $94.2$& $92.8$ & $95.1 \pm 1.0$\\
\midrule
\end{tabular}
\caption{Test set accuracies of 10 runs at difference stages of training. VMT@40K denotes the accuracy of VMT at iteration $40000$. DIRT-T takes the VMT@40K model as the initialization model. Our model shows small variance in performance for the task with instance normalization, while shows moderate variance for the task without instance normalization.} 
\label{table:m2s}
\end{table*}

\begin{table*}[t]
\centering
\begin{tabular}{c|cccccccccc|c}
\midrule
\multicolumn{11}{c|}{MNIST $\rightarrow$ SVHN without Instance-Normalized Input:}&Average\\
 \midrule
 $\text{VMT}_{\text{prob}}$@40K & $47.5$ & $54.0$ & $50.1$ & $49.6$ & $50.5$ & $50.2$ & $44.8$& $48.1$& $38.4$& $43.1$ & $47.6 \pm 4.5$\\
\midrule
 $\text{VMT}_{\text{prob}}$@160K & $59.2$ & $62.3$ & $61.0$ & $60.0$ & $60.3$ & $65.3$ & $57.0$& $60.4$& $21.7$& $51.1$ & $55.8 \pm 12.5$\\
\midrule
 $\text{VMT}_{\text{prob}}$ + DIRT-T@160K & $95.0$ & $94.4$ & $93.3$ & $91.9$ & $89.7$ & $88.5$ & $78.3$& $77.6$& $19.6$& $15.6$ & $74.4 \pm 30.6$\\
\midrule
\end{tabular}
\caption{Test set accuracies of VMT with mixup-on-probabilities. The results indicate high variance in performance. The model sometimes behaves unstably and collapses to a degenerate solution.} 
\label{table:prob}
\end{table*}

\subsection{Implementation Detail}

\subsubsection{Architecture} 
We use the same network architectures as the ones in VADA\cite{Shu2018} for a fair comparison. In particular, a small CNN is used for the tasks of digits, and a larger CNN is used for the tasks between CIFAR-10 and STL-10. 

\subsubsection{Baselines}
We primarily compare our model with two baselines: VADA \cite{Shu2018} and Co-DA \cite{Kumar2018}. Also based on VADA, Co-DA used a co-regularization method to make a better domain alignment. We also show the results of several other recently proposed unsupervised domain adaptation models for comparison.

\subsubsection{Others}
Following \cite{Shu2018}, we replace gradient reversal \cite{Ganin2016} with the adversarial training \cite{Goodfellow2014} of alternating updates between the domain discriminator and feature encoder. We also follow \cite{Shu2018} to apply the instance normalization to the input images, and report the results of using or without using the instance normalization. The implementation of our model is based on the official implementation of VADA.

\subsection{MNIST $\rightarrow$ SVHN}
We first evaluate VMT on the adaptation task from MNIST to SVHN, which is usually regarded as a challenging task \cite{Ganin2016,Shu2018}. It is especially difficult when the input is not instance-normalized.

\subsubsection{Without Instance Normalization}
When not applying instance normalization, VADA removes the conditional entropy minimization during training, as it behaves unstably and finds a degenerate solution quickly \cite{Shu2018}. We find that this problem no longer exists in our model, and thus we keep the conditional entropy minimization during training. As Table \ref{table:evaluation} shows, when not applying DIRT-T, VMT outperforms VADA and Co-DA by 11.8\% and 4\%, respectively. When applying DIRT-T, VMT outperforms VADA and Co-DA by $33.6\%$ and $25.1\%$, respectively. 

The reported accuracies are averaged over 10 trials, and we list the complete results in Table \ref{table:m2s}. We have the following three observations from Table \ref{table:m2s}. First, the results indicate moderate variance in performance, but even in the worst-case (76.6\%), it still outperforms VADA and Co-DA by 22.1\% and 13.6\%. Second, the best model can achieve an accuracy of $95.9\%$, which is very close to the one applying instance normalization to the input. Third, generally speaking, if the model shows good performance at iteration $40000$, it usually can achieve good results for both the VMT model and DIRT-T model. 

Note that the reported results of VMT and DIRT-T are both achieved at iteration $160000$. Interestingly, if we continue to train the DIRT-T model, we observe an accuracy of $96.4\%$ at iteration $260000$. Moreover, we train a classifier on the target domain (i.e., SVHN) with labels revealed using the same network architecture and same settings, and it is usually treated as an upper bound for domain adaptation models. This train-on-target model achieves an accuracy of 96.5\%. Our model achieves a very close accuracy (96.4\%) to the upper bound (96.5\%).

The reported results are achieved under $\lambda_t=0.01$. If we set $\lambda_t$ to $0.02$, VMT@160K generally performs better and can achieve an accuracy of $73.7\%$ but sometimes collapses to a degenerate solution. More discussion about the performance of  $\lambda_t=0.02$ is provided in the Appendix.

\subsubsection{With Instance Normalization}
When applying instance normalization, \mbox{VMT + DIRT-T} outperforms VADA and Co-DA by $18.6\%$ and $7.1\%$, respectively. Moreover, we also show the results of 10 runs in Table \ref{table:m2s}. Different from the case without instance normalization, the results of applying instance normalization show small variance in performance. Note that compared with VADA, we set a larger $\lambda_t=0.06$, which is able to improve the performance, while VADA will collapse to a degenerate solution if we set the same value of $\lambda_t$. We show the comparison results of VMT between $\lambda_t=0.01$ and $\lambda_t=0.06$ in the Appendix. Similar to the case without instance normalization, if we continue to train the DIRT-T model, VMT + DIRT-T can achieve an accuracy of $96.4\%$ at iteration $280000$.

\begin{figure*}[t]
\centering
\begin{tabular}{ccccc}

 \includegraphics[width=0.175\textwidth]{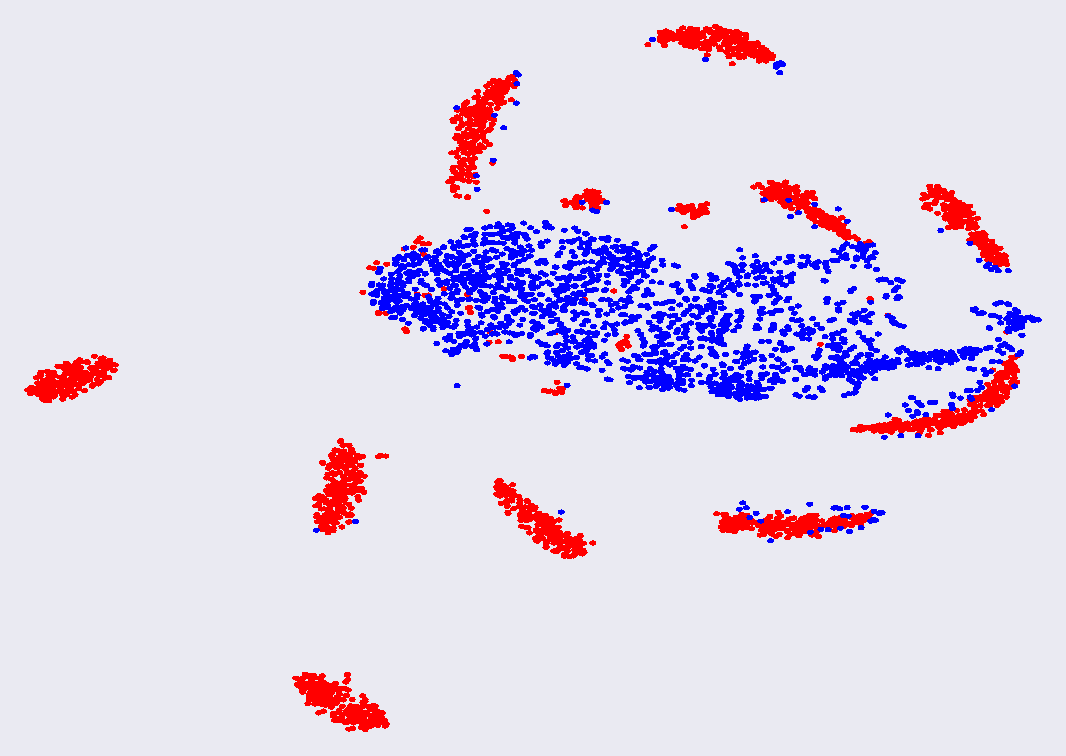}
 &
 \includegraphics[width=0.175\textwidth]{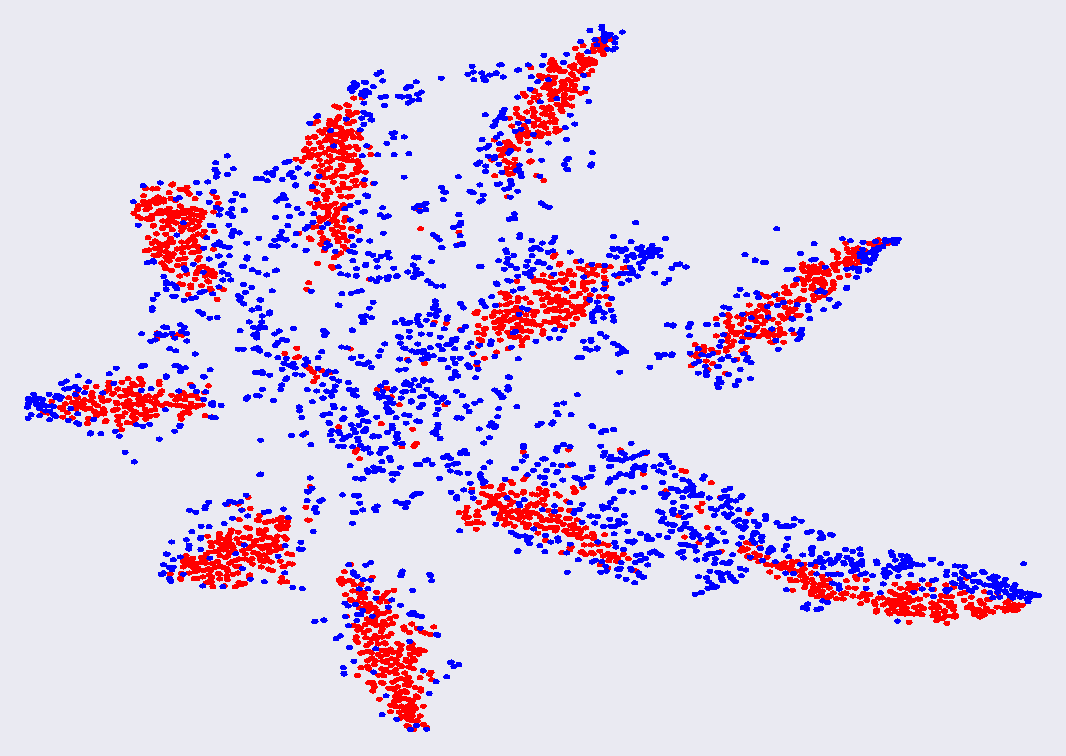}
  &
 \includegraphics[width=0.175\textwidth]{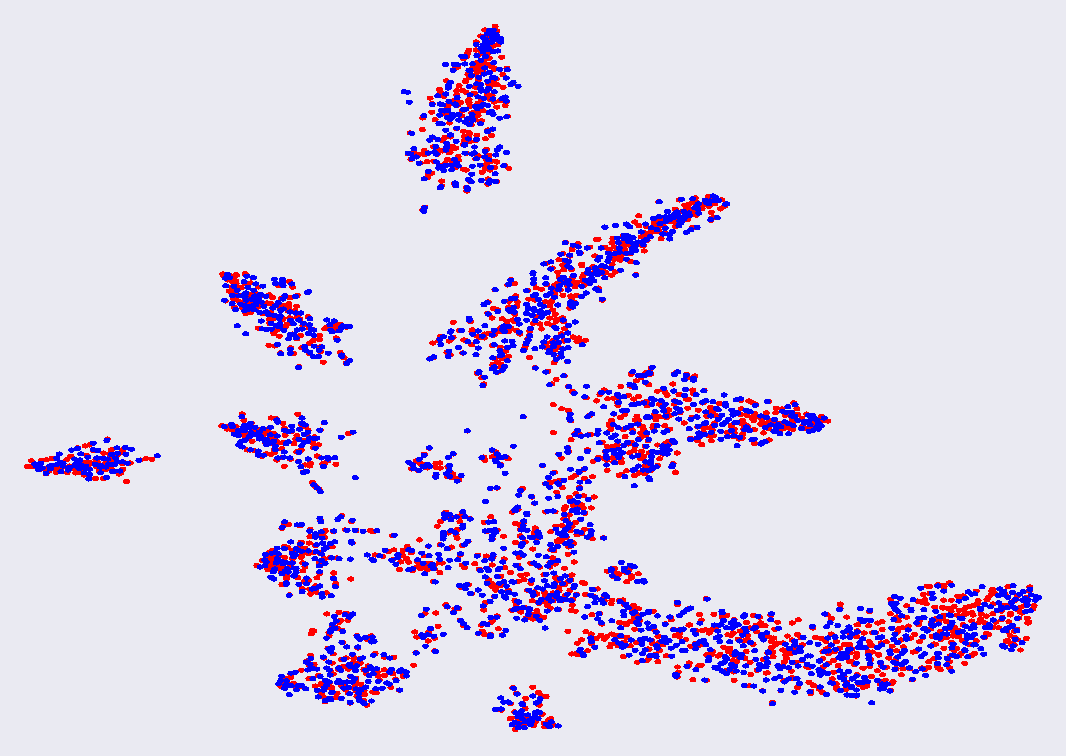}
 &
 \includegraphics[width=0.175\textwidth]{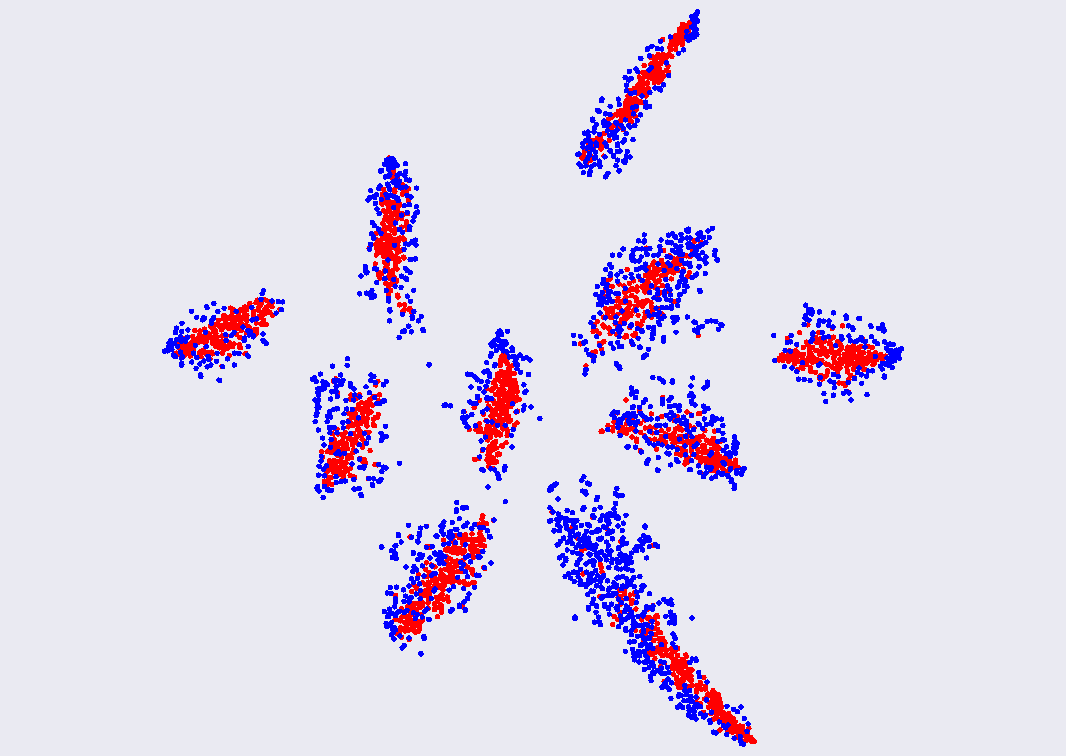}
 &
 \includegraphics[width=0.175\textwidth]{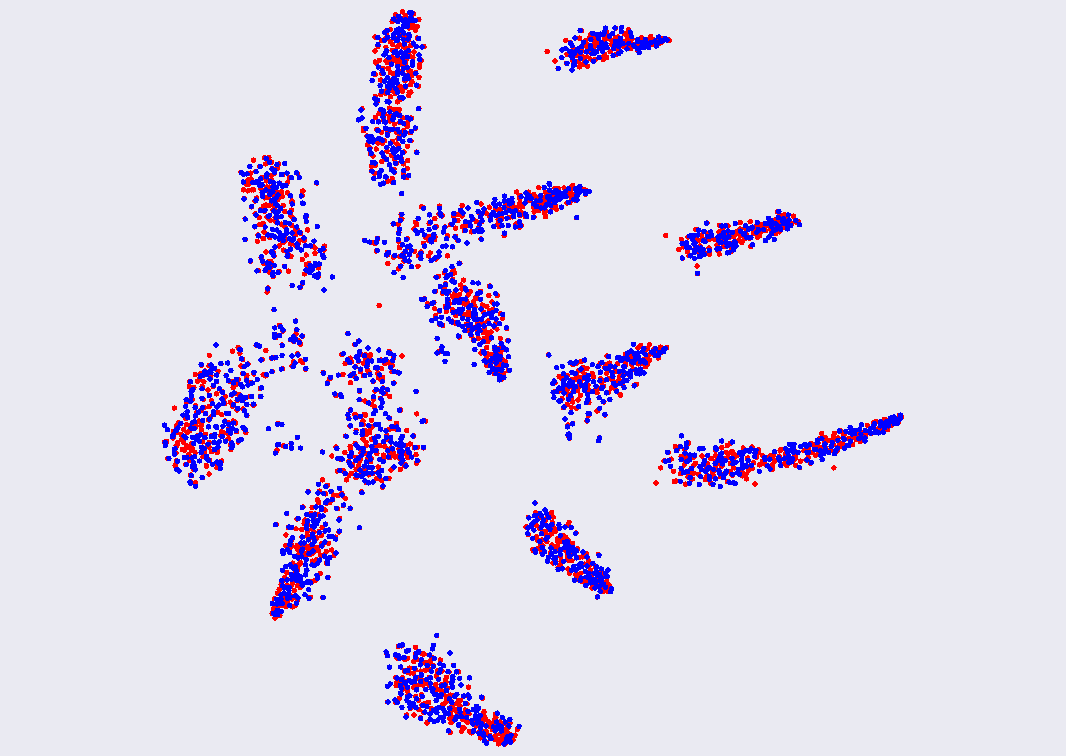}
\\
(a) Source only & (b) VADA & (c) VADA + DIRT-T & (d) VMT & (e) VMT + DIRT-T
\end{tabular}
\caption{
T-SNE visualization of the last hidden layer for MNIST (red) to SVHN (blue) without instance normalization. Compared with VADA, VMT generates closer features vectors for the source and target domains, and shows stronger clustering performance for the target domain. VMT + DIRT-T makes the source and target features closest.
}
\label{fig:visualization}
\end{figure*}

\subsection{Other Digits Adaptation Tasks}
For the other digits adaptation tasks including SVHN $\rightarrow$ MNIST, MNIST $\rightarrow$ MNIST-M, and SYN DIGITS $\rightarrow$ SVHN, the baselines already achieve high accuracies. As Table \ref{table:evaluation} shows, for these tasks, VMT still outperforms VADA and shows similar performance to Co-DA.

\subsection{CIFAR-10 $\leftrightarrow$ STL-10.} 
For CIFAR-10 and STL-10, there are nine overlapping classes between the two datasets. Following \cite{French2018,Shu2018,Kumar2018}, we remove the non-overlapping classes and remain the nine overlapping classes. STL-10 $\rightarrow$ CIFAR-10 is more difficult than CIFAR-10 $\rightarrow$ STL-10, as STL-10 has less labeled samples than CIFAR-10. We observe more significant gains for the harder task of STL-10 $\rightarrow$ CIFAR-10. VMT outperforms VADA by $8.1\%$ and $6.7\%$, and outperforms Co-DA by $4.8\%$ and $3.8\%$, for with and without instance normalization, respectively. For CIFAR-10 $\rightarrow$ STL-10, VMT improves over VADA by $3\%$ and $2\%$ for with and without instance normalization, respectively. Note that DIRT-T has no effect on this task, because STL-10 contains a very small training set, making it difficult to estimate the conditional entropy.

\subsection{Mixup on Logits or Probabilities?}
The inputs and outputs of the softmax layer are called logits and probabilities, respectively. In this experiment, we compare the performance of two schemas: mixup-on-logits (Eq. \ref{eq:vmt_logits}) and mixup-on-probabilities (Eq. \ref{eq:vmt}). The results of mixup-on-logits and mixup-on-probabilities are shown in Table \ref{table:m2s} and Table \ref{table:prob}, respectively. We can observe that mixup-on-logits performs better and more stable than mixup on probabilities. Mixup-on-probabilities sometimes collapses to a degenerate solution whose accuracy is only $15.6\%$. Moreover, we have also investigated the performance of mixup on some intermediate layers in the Appendix, and mixup-on-logits achieves the best performance.

\subsection{Comparing with Virtual Adversarial Training}
\label{sec:cmp_vat}
Virtual adversarial training (VAT) \cite{Miyato2018} is another approach to impose the locally-Lipschitz constraint, as used in VADA. We conduct comparison experiments between VAT (Eq. \ref{eq:vat}) and our proposed VMT (Eq. \ref{eq:vmt_loss}), and the results are shown in Table \ref{table:cmp_vat}. VMT achieves higher accuracies than VAT for all the tasks, which demonstrates that VMT surpasses VAT in favoring the cluster assumption. Furthermore, combining VMT and VAT is able to further improve the performance except for CIFAR-10 $\rightarrow$ STL-10. This shows that VMT is orthogonal to VAT for most tasks, and they can be used together to constrain the local Lipschitzness. Compared with VAT, another advantage of VMT is the low computational cost. For the task of MNIST $\rightarrow$ SVHN with instance normalization, VMT costs about 100 seconds for 1000 iterations in our GPU server, while VAT needs about 140 seconds. The dynamic comparison of the accuracy over time is presented in the Appendix.

\begin{table}[t]
\centering
\small
\setlength\tabcolsep{1.7pt}
\begin{tabular}{c|cccccc}
\toprule
\multicolumn{1}{c|}{Source}  & MNIST & SVHN  & MNIST               & SYN    & CIFAR & STL\\
\multicolumn{1}{c|}{Target}  & SVHN  & MNIST & MNISTM             & SVHN   & STL   & CIFAR\\
\midrule
\multicolumn{7}{c}{With Instance-Normalized Input:}\\
\midrule
 $\L_c$                  & $66.8$ & $83.1$ & $93.8$              & $93.4$ & $79.1$         & $68.6$\\
\midrule
 $\L_c, \L_v$            & $73.3$ & $94.5$ & $95.7$             & $94.9$ & $78.3$         & $71.4$ \\
\midrule
 $\L_c, \L_m$            & $82.6$ & $98.5$    & $97.3$                 & $95.6$ & $\bf{82.0}$ & $78.3$  \\
 \midrule
 $\L_c, \L_v, \L_m$      & $\bf{85.2}$&$\bf{98.9}$&$\bf{98.0}$ &$\bf{96.4}$&$81.3$     & $\bf{79.5}$  \\
\midrule
\end{tabular}
\caption{Test set accuracy in comparison experiments between VAT and VMT.  $\L_c$ denotes the conditional entropy loss, $\L_v$ denotes the VAT loss, and $\L_m$ denotes the VMT loss. $\{\L_c, \L_m\}$ means that we only use $\L_{y,d}$, $\L_c$, and $\L_m$ in Eq. \ref{eq:final_obj}, and set the weights of the other losses to 0. The results of $\{\L_c\}$ and $\{\L_c, \L_v\}$ are duplicated from \cite{Shu2018}.} 
\label{table:cmp_vat}
\end{table}

\subsection{Visualization of Representation}
We further present the T-SNE visualization results in Figure \ref{fig:visualization}. We use the most challenging task (i.e., MNIST $\rightarrow$ SVHN without instance normalization) to highlight the differences. As shown in Figure \ref{fig:visualization}, source-only training shows a discriminative clustering result for the source domain but generates only one cluster for the target domain. We can observe that VMT makes the features from the source and target domains much closer than VADA, and shows stronger clustering performance of the target samples than VADA does. VMT + DIRT-T can further get closer feature vectors of the source and target domains.

\section{Conclusion}
In this paper, we proposed a novel method called Virtual Mixup Training (VMT), for unsupervised domain adaptation. VMT is designed to constrain the local Lipschitzness to favor the cluster assumption. The idea of VMT is to make linearly-change predictions along the lines between pairs of training samples. We empirically show that VMT significantly improves the performance of the recent state-of-the-art model called VADA. For a challenging task of adapting MNIST to SVHN, VMT can outperform VADA by over 33.6\%.


\bibliography{aaai20}
\bibliographystyle{aaai}
\clearpage
\appendix
\onecolumn

\section{Appendix}

\subsection{Proof about the Linear Behavior of VMT}
\begin{lemma}
Optimizing Eq. \ref{eq:vmt_loss} encourages the classifier $f$ to behave linearly between $x_i$ and $x_j$.
\end{lemma}

\begin{proof}
Minimizing $\Expect_{x \sim \X} \brac{ \KL(\tilde{y} \| f(\tilde{x}))}$ enforces the classifier $f$ to output $\tilde{y}$ for $\tilde{x}$. If $\Expect_{x \sim \X} \brac{ \KL(\tilde{y} \| f(\tilde{x}))} = 0$, then $(\tilde{x}, \tilde{y})$ becomes a point on the classifier $f$.
\begin{equation}
\label{eq:proof}
\begin{split}
&\frac{\tilde{y}-y_1}{\tilde{x}-x_1}=\frac{(1-\lambda)(y_1-y_2)}{(1-\lambda)(x_1-x_2)}=\frac{y_1-y_2}{x_1-x_2}, \\
&\frac{\tilde{y}-y_2}{\tilde{x}-x_2}=\frac{\lambda(y_1-y_2)}{\lambda(x_1-x_2)}=\frac{y_1-y_2}{x_1-x_2}.
\end{split}
\end{equation}
Eq. \ref{eq:proof} indicates the same slope among $x_1$, $\tilde{x}$, and $x_2$.
Therefore, optimizing $\Expect_{x \sim \X} \brac{ \KL(\tilde{y} \| f(\tilde{x}))}$ encourages the classifier $f$ to behave linearly between $x_i$ and $x_j$ and we finish the proof.

\end{proof}

\subsection{Proof about Imposing Locally-Lipschitz Constraint of VMT}
\begin{lemma}
If $f$ is linear in all directions at $\tilde{x}$, then $f$ is locally-Lipschitz at $\tilde{x}$.
\end{lemma}

\begin{proof}
Since $f$ is linear in all directions at $\tilde{x}$, we can assume a largest slope $M$ for the surrounding points $\{x \in \X \ \big|\ ||x-\tilde{x}|| \leq \delta_0 \}$ to $\tilde{x}$ such that

\begin{equation}
\bigg|\bigg|\frac{f(x)-f(\tilde{x})}{x-\tilde{x}} \bigg|\bigg|\leq M. 
\end{equation}
Then we can get
\begin{equation}
||f(x)-f(\tilde{x})|| \leq M ||x-\tilde{x}||,
\end{equation}
which implies the local Lipschitzness at $\tilde{x}$ as in Definition \ref{def:local}, and we finish the proof.

\end{proof}

\subsection{Hyperparameters}
\label{sec:hyper}

We restrict the hyperparameter search to $\lambda_d = \sset{0, 0.01}$, $\lambda_s = \sset{0, 1}$, $\lambda_t = \sset{0.01,0.02,0.04,0.06,0.08,0.1,1}$, and $\beta = \sset{0.001,0.01,0.1,1}$. Compared with VADA, the main different setting is $\lambda_t$. We empirically find that increasing the value of $\lambda_t$ is able to improve the performance significantly. But for VADA, it will collapse to a degenerate solution if we set the same value of $\lambda_t$. We set the refinement interval \cite{Shu2018} of DIRT-T to 5000 iterations. The only exception is MNIST $\rightarrow$ MNIST-M. For this special case, we set the refinement interval to 500, and set the weight of $\L_m(f; \X_t)$ to $10^{-3}$. We use Adam Optimizer (learning rate $=0.001$, $\beta_1 = 0.5$, $\beta_2=0.999$) with an exponential moving average (momentum $= 0.998$) to the parameter trajectory. When not applying DIRT-T, we train VMT for $\sset{40000, 80000, 160000}$ iterations, with the number of iterations chosen as a hyperparameter. When applying DIRT-T, we train VMT for $\sset{40000, 80000}$ iterations as the initialization model, and train DIRT-T for $\sset{40000, 80000, 160000}$ iterations.

\begin{table}[h]
\centering
\begin{tabular}{l|ccccc}
\toprule
Task & Instance Normalization                 & $\lambda_d$     & $\lambda_s$    & $\lambda_t$     & $\beta$ \\
\midrule
MNIST $\rightarrow$ SVHN    & Yes     & $0.01$       & $1$             & $0.06$     & $0.01$ \\
MNIST $\rightarrow$ SVHN    & No     & $0.01$     & $1$         & $0.01$            & $0.001$ \\
SVHN $\rightarrow$ MNIST    & Yes, No&  $0.01$     & $0$        & $0.1$               & $0.01$ \\
MNIST $\rightarrow$ MNIST-M & Yes, No& $0.01$     & $0$             & $0.01$        & $0.01$ \\
SYN $\rightarrow$ SVHN       & Yes, No& $0.01$     & $1$       & $1$                    & $1$ \\
CIFAR $\rightarrow$ STL     & Yes, No& $0$             & $1$     & $0.1$              & $0.01$ \\
STL $\rightarrow$ CIFAR     & Yes, No& $0$             & $0$             & $0.1$      & $0.01$ \\
\midrule
\end{tabular}
\caption{
List of the hyperparameters.
}
\label{table:hyper}
\end{table}

\subsection{The Effect of $\lambda_t$ for MNIST $\rightarrow$ SVHN without Instance Normalization}
In this experiment, we compare the performance between $\lambda_t=0.02$ and $\lambda_t=0.01$ for MNIST $\rightarrow$ SVHN without Instance Normalization. We have the following four observations from Table \ref{table:m2s_wo_lambda}. First, $\lambda_t=0.02$ sometimes collapses to a degenerate solution whose accuracy is less than $20\%$. Second, when the model with $\lambda_t=0.02$ collapses, the test set accuracy on the source domain is also very small. Thus we can exclude the collapsed model at the early stage of training by checking the test set accuracy on the source domain. Third, when not applying DIRT-T (VMT@160K in the table), $\lambda_t=0.02$ outperforms $\lambda_t=0.01$ significantly except for the case that $\lambda_t=0.02$ collapses to a degenerate solution. Forth, when applying DIRT-T, $\lambda_t=0.02$ and $\lambda_t=0.01$ have similar performance.

\begin{table*}[h]
\small
\centering
\begin{tabular}{c|cccccccccc|c}
\midrule
\multicolumn{11}{c}{MNIST $\rightarrow$ SVHN without Instance-Normalized Input:}\\
\midrule
  VMT@40K:Source ($\lambda_t=0.02$)& $95.2$ & $96.5$ & $98.1$ & $96.7$ & $98.5$ & $98.3$& $98.4$& $98.6$& $98.5$& $10.4$ & $88.9 \pm 27.6$\\
 \midrule
 VMT@40K ($\lambda_t=0.02$)        & $60.6$ & $60.5$ & $58.8$ & $47.3$ & $59.9$ & $53.3$& $55.4$& $52.8$& $47.0$& $14.6$ & $51.0 \pm 13.8$\\
 \midrule
 VMT@160K ($\lambda_t=0.02$)       & $73.4$ & $73.7$ & $68.5$ & $54.3$ & $67.0$ & $69.8$& $69.4$& $67.1$& $65.5$& $15.6$ & $62.4 \pm 17.3$\\
\midrule
VMT + DIRT-T@160K ($\lambda_t=0.02$)& $96.0$ & $95.5$ & $95.4$ & $92.0$ & $87.3$ & $80.1$& $79.6$& $74.4$& $73.7$& $15.6$ & $79.0 \pm 23.9$\\
  \midrule
 VMT@40K ($\lambda_t=0.01$)         & $57.7$ & $57.6$ & $57.4$ & $50.1$ & $54.1$ & $50.0$& $52.7$& $40.1$& $53.9$& $45.5$ & $51.9 \pm 5.7$\\
\midrule
 VMT@160K ($\lambda_t=0.01$)        & $64.6$& $64.6$ & $58.3$ & $61.4$ & $57.3$ & $63.7$& $64.5$& $49.4$& $56.9$& $52.0$ & $59.3 \pm 5.5$\\
\midrule
 VMT + DIRT-T@160K ($\lambda_t=0.01$)& $95.9$ & $95.9$ & $95.3$ & $95.3$ & $95.1$ & $83.6$& $82.8$& $80.4$& $79.9$& $76.6$ & $88.1 \pm 8.0$\\
\midrule
\end{tabular}
\caption{Comparison between $\lambda_t=0.02$ and $\lambda_t=0.01$ for MNIST $\rightarrow$ SVHN without Instance-Normalized Input. VMT@40K:Source denotes the test set accuracy of VMT on the source domain at iteration $40000$. } 
\label{table:m2s_wo_lambda}
\end{table*}

\subsection{The Effect of $\lambda_t$ for MNIST $\rightarrow$ SVHN with Instance Normalization}
Compared with VADA, we increase the value of $\lambda_t$ for several tasks, as shown in Table \ref{table:hyper}. Because we find that increasing the value of $\lambda_t$ can improve the performance significantly. But for VADA, it will collapse to a degenerate solution quickly if we set the same value of $\lambda_t$. Table \ref{table:m2s_lambda} shows the results of $\lambda_t=0.06$ and $\lambda_t=0.01$ for MNIST $\rightarrow$ SVHN with instance normalization, and we can observe that $\lambda_t=0.06$ outperforms $\lambda_t=0.01$ significantly.

\begin{table*}[h]
\small
\centering
\begin{tabular}{c|cccccccccc|c}
\midrule
\multicolumn{11}{c|}{MNIST $\rightarrow$ SVHN with Instance-Normalized Input:} & Average\\
\midrule
 VMT@40K($\lambda_t=0.06$)         & $86.1$ & $85.6$ & $86.1$ & $85.8$ & $85.2$ & $85.6$& $84.3$& $84.0$& $84.5$& $84.3$ & $85.2 \pm 0.8$\\
\midrule
 VMT + DIRT-T@160K($\lambda_t=0.06$) & $96.0$ & $95.9$ & $95.6$ & $95.5$ & $95.4$ & $95.3$& $95.3$& $95.1$& $94.2$& $92.8$ & $95.1 \pm 1.0$\\
\midrule
 VMT@40K($\lambda_t=0.01$)         & $72.8$ & $75.3$ & $74.8$ & $74.4$ & $73.7$ & $72.3$& $72.3$& $73.8$& $73.9$& $71.7$ & $73.5 \pm 1.2$\\
\midrule
 VMT + DIRT-T@160K($\lambda_t=0.01$) & $88.8$ & $87.3$ & $86.4$ & $85.4$ & $85.4$ & $84.4$& $83.8$& $83.4$& $82.0$& $79.9$ & $84.7 \pm 2.6$\\
\midrule
\end{tabular}
\caption{Comparison between $\lambda_t=0.06$ and $\lambda_t=0.01$ for MNIST $\rightarrow$ SVHN with Instance-Normalized Input.} 
\label{table:m2s_lambda}
\end{table*}

\subsection{Comparing with Mixup on Intermediate Layers}
In this experiment, we investigate the performance of mixup on intermediate layers. Table \ref{table:intermediate} shows the performance of mixup on the input of the last fully-connected layer. We also evaluated other intermediate layers, and mixup on logits performs the best.

\begin{table*}[h]
\centering
\begin{tabular}{c|cccccccccc|c}
\midrule
\multicolumn{11}{c|}{MNIST $\rightarrow$ SVHN with Instance-Normalized Input:} & Average\\
\midrule
  $\text{VMT}_{\text{logits}}$         & $86.1$ & $85.6$ & $86.1$ & $85.8$ & $85.2$ & $85.6$& $84.3$& $84.0$& $84.5$& $84.3$ & $85.2 \pm 0.8$\\
\midrule
  $\text{VMT}_{\text{logits}}$ + DIRT-T & $96.0$ & $95.9$ & $95.6$ & $95.5$ & $95.4$ & $95.3$& $95.3$& $95.1$& $94.2$& $92.8$ & $95.1 \pm 1.0$\\
\midrule
  $\text{VMT}_{\text{inter}}$         & $85.5$ & $83.9$ & $84.0$ & $84.1$ & $77.9$ & $74.2$& $70.8$& $47.4$& $32.4$& $16.3$ & $65.7 \pm 24.8$\\
\midrule
  $\text{VMT}_{\text{inter}}$ + DIRT-T & $95.9$ & $94.9$ & $94.2$ & $94.1$ & $87.1$ & $82.7$& $75.1$& $51.9$& $32.6$& $17.6$ & $72.6 \pm 28.6$\\
\midrule
\end{tabular}
\caption{Comparison between mixup on logits and mixup on an intermediate layer.} 
\label{table:intermediate}
\end{table*}

\clearpage

\subsection{Dynamic Accuracy Results of VAT and VMT}
Compared with VAT, one advantage of VMT is the low computational cost. Figure \ref{fig:time} shows the dynamic results of the accuracy over time. We can observe that the model trained with VMT increases the accuracy much faster than the one trained with VAT.

 \begin{figure}[h]
\centering
 \includegraphics[width=0.4\textwidth]{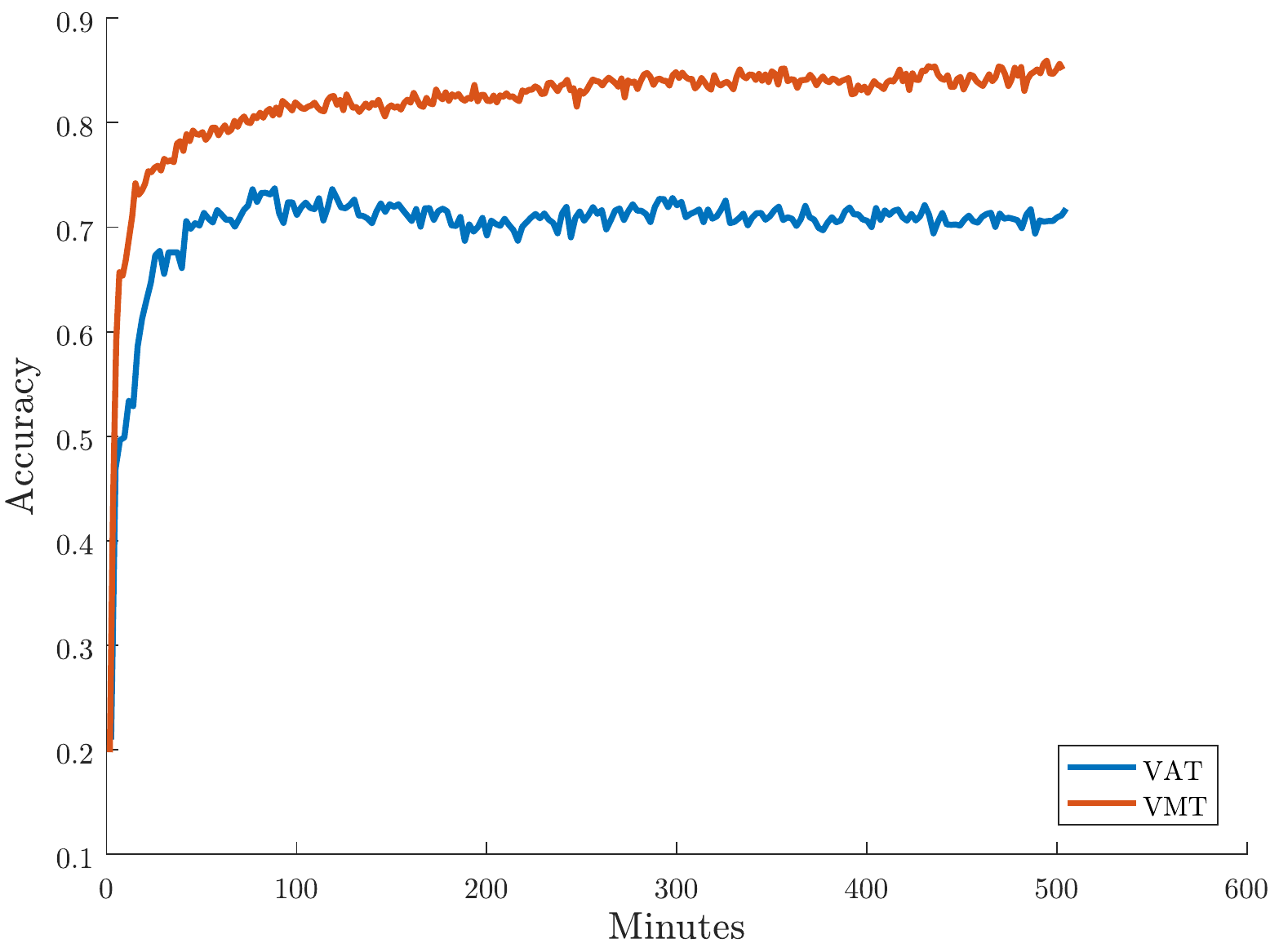}
\caption{
Dynamic test set accuracy on the adaptation task of MNIST $\rightarrow$ SVHN with instance normalization. The blue line is for VAT and the red line is for VMT.
}
\label{fig:time}
\end{figure}

\end{document}